%% file: neurips_2026.tex
\documentclass{article}

\usepackage[preprint]{neurips_2026}

\usepackage[utf8]{inputenc} 
\usepackage[T1]{fontenc}    
\usepackage{hyperref}       
\usepackage{url}            
\usepackage{booktabs}       
\usepackage{amsfonts}       
\usepackage{nicefrac}       
\usepackage{microtype}      
\usepackage{xcolor}         

\usepackage{amsmath}
\usepackage{amssymb}
\usepackage{mathtools}
\usepackage{amsthm}
\usepackage{algorithm}
\usepackage{algorithmic}
\usepackage{subcaption}

\usepackage[capitalize,noabbrev]{cleveref}

\theoremstyle{plain}
\newtheorem{theorem}{Theorem}[section]

\newtheorem{lemma}[theorem]{Lemma}

\theoremstyle{definition}
\newtheorem{definition}[theorem]{Definition}

\theoremstyle{remark}

\usepackage[textsize=tiny, disable]{todonotes}

\title{BOIL: Learning Environment Personalized Information}

\author{%
  Rohan Patil \quad\quad Henrik I. Christensen\\
  Department of Computer Science and Engineering\\
  University of California San Diego\\
  San Diego, CA 92093 \\
  \texttt{\{rpatil, hichristensen\}@ucsd.edu} \\
}

\begin{document}

\maketitle

\begin{abstract}
Navigating complex environments poses challenges for multi-agent systems, requiring efficient extraction of insights from limited information. In this paper, we introduce the Blackbox Oracle Information Learning (BOIL) process, a scalable solution for extracting valuable insights from the environment structure. Leveraging the Pagerank algorithm and common information maximization, BOIL facilitates the extraction of information to guide long-term agent behavior applicable to problems such as coverage, patrolling, and stochastic reachability. Through experiments, we demonstrate the efficacy of BOIL in generating strategy distributions conducive to improved performance over extended time horizons, surpassing heuristic approaches in complex environments. \url{https://anonymous.4open.science/r/boil-DFF0}
\end{abstract}

\input{sections/introduction}
\input{sections/related}

\input{sections/preliminaries}
\input{sections/boil}
\input{sections/results}
\input{sections/discussion}

\bibliographystyle{icml2026}
\bibliography{references}

\newpage
\appendix
\input{appendix/probability_sandwich}    
\input{appendix/finegrained_estimator}
\input{appendix/large_environment}


\end{document}

%% file: sections/introduction.tex
\section{Introduction}
\label{sec:introduction}

Since the 1980s, the concept of employing multiple interacting agents for various tasks has garnered attention in the fields of robotics, artificial intelligence, and distributed systems. Previous research endeavors have demonstrated successful methods for utilizing multiple robots to tackle diverse challenges such as area coverage, patrolling, and reachability
~\cite{bierkens2016non,clempner2018continuous,diaz2023distributed,hespanha1999multiple,langley2021heterogeneous,rahili2017distributed,stern2006multiagent}. These efforts embraced a multitude of approaches spanning game theory, genetic algorithms, greedy heuristics, and more, reflecting the multifaceted nature of the problems.

Despite this breadth of exploration, a fundamental trade-off persists between achieving optimal solutions, maintaining computational tractability, and ensuring scalability in the number of agents or the size of the environment. Some methodologies adopt an independence assumption among agents to facilitate tractability for larger teams
\cite{diaz2023distributed,langley2021heterogeneous},  sacrificing potential optimality for scalability. The main objective here 
is fine-grained control of these trade-offs. 
We operate under the assumption of having access to an oracle whose information is indirectly accessible, and we aim to devise a computationally scalable approach to extract insights from this oracle. Our focus lies in demonstrating the feasibility of extracting information from an oracle whose behavior adapts to environmental changes. Leveraging the Pagerank algorithm \cite{page1998pagerank}, renowned for its computational efficiency, we propose a scalable method for extracting information from the oracle. Crucially, our approach remains independent of the number of agents involved, and we formulate the problem as a common information maximization task~\footnote{There are multiple ways to define common information. We use the definition given by~\citet{liu2010common}.}.


We begin by providing a comprehensive review of existing literature concerning the aforementioned tasks, followed by an introduction to Non-reversible Markov chains and Supervised PageRank techniques. Subsequently, we illustrate the construction of our proposed process within the context of the coverage problem, termed BOIL (Blackbox Oracle Information Learning), emphasizing its ability to distill information from the oracle into learnable parameters. Many real world problems like mobile sensor coverage~\cite{rahili2017distributed}, forest fire detection~\cite{alsammak2022use}, agricultural monitoring~\cite{albani2017monitoring}, traffic data collection~\cite{elloumi2018monitoring}, etc. can be modeled using coverage. In many real-world scenarios, the number of agents $n$ is significantly smaller than what is required for static coverage. Specifically, we operate in the sparse-agent regime, where the collective sensing footprint of all agents at any time $t$ covers only a fraction of the environment. In such settings, static equilibrium methods are physically incapable of solving the problem, necessitating a long-term dynamic strategy that optimizes the time-averaged visit frequency across the domain. With the foundational structure of BOIL established, we present results that demonstrate its efficacy, particularly in scenarios where a higher number of parameters is computationally feasible. Furthermore, we discuss how our method can be used to address patrolling and reachability tasks. Specifically, BOIL is designed for static or slowly changing environments, such as infrastructure inspection, environmental monitoring, and warehouse patrolling. While handling rapidly time-varying or adversarial environments is an important future direction, it is orthogonal to our current focus on learning environment-level information structure, allowing us to prioritize computational efficiency and accessibility in the optimization process.

The primary contribution lies in showcasing the utility of extracted information from a blackbox oracle, offering theoretical insights and practical implications for enhancing the performance and robustness of multi-agent systems in complex environments. We run simulated experiments to support the claims about the effectiveness of the process.

%% file: sections/related.tex
\section{Related Works}
\label{sec:related_works}


%


Several approaches have been proposed in the literature to address the challenges of coverage, patrolling, and reachability in multi-agent systems. \citet{stern2006multiagent} introduced a genetic algorithm-based technique for coverage, optimizing a complex fitness function dependent on factors such as distance, travel time, and visibility within a discretized space. While providing complete trajectories, this approach becomes computationally intractable for large environments or a high number of agents due to its dependency on the number of agents. \citet{rahili2017distributed} presented a game-theoretic framework for distributed coverage of mobile sensors, aiming to find equilibrium positions for sensors using reinforcement learning to converge to Nash equilibrium. This approach incorporates utility functions for agents directly into transition probabilities, assuming reversible agent movements~\footnote{Reversible movements implies that if an agent can take action $a$ to transition from state $s$ to $s'$, then there must exist an action $a'$ such that the agent can transition from $s'$ back to $s$.} and focusing on reaching equilibrium via agent interactions. Our work generalizes to non-reversible agent movements. In both works, the planning and control of agents is tightly coupled making it hard to tweak individual components separately. In this work, we provide a way to loosely couple planning and control allowing a fine-grained ability to deal with various trade-offs.

\citet{mathew2011metrics}~showed that it is possible to create control inputs for a group of agents based on a probabilistic high level plan. Leveraging Ergodic control, they addressed the coverage problem with a given distribution for space coverage, defining first and second-order dynamics of agents. Their work shows that we can create a loose coupling between planning and control, the main focus is on control. In contrast, our work focuses on the planning part of the problem. \citet{abraham2018decentralized} demonstrated a decentralized ergodic control method for coverage, focusing on achieving a given target distribution. They claim that their work can be extended to pursuit-evasion games but do not provide a mathematical formulation. Note that ergodic control methods~\cite{mathew2011metrics,abraham2018decentralized} fundamentally require a target spatial distribution as prior input. They solve the `tracking' problem (how to move to match a distribution), whereas BOIL solves the `generation' problem (it generates the target distribution). Consequently, we do not compare quantitatively against ergodic controllers, as BOIL provides the necessary input they require to function in these environments. 

Stochastic reachability, akin to the work of \citet{hespanha1999multiple}, pertains to solving agent trajectories to maximize the probability of reaching specific locations within a fixed time. We discuss how it is possible to provide a mathematical formulation for stochastic reachability with minor changes. Patrolling tasks present distinct challenges, where the number of objective points may be significantly fewer than the total environment size. Game-theoretic approaches, such as Stackelberg games \cite{clempner2018continuous, gan2018stackelberg} and Nash equilibrium strategies \cite{langley2021heterogeneous}, have been widely adopted. Additionally, constraint-based methods \cite{diaz2023distributed} have been explored, offering scalable solutions for multi-agent patrolling.

%% file: sections/preliminaries.tex
\section{Preliminaries}
\label{sec:preliminaries}

\label{subsec:nonreversible_markov_chain}

\textbf{Non-Reversible Markov Chain.} A Markov chain over state space $S$ has a stationary distribution $\pi$ satisfying the 
\textit{global balance} condition $\sum_{u \in S} \pi(u)P(u\!\rightarrow\! v) = \pi(v)$, 
which is weaker than \textit{detailed balance} $\pi(u)P(u\!\rightarrow\! v) = 
\pi(v)P(v\!\rightarrow\! u)$. Chains satisfying only global balance are non-reversible. 
While Metropolis-Hastings (MH)~\citep{metropolis1953equation, hastings1970monte} provides 
sampling guarantees for reversible chains, \citet{bierkens2016non} extends this to the 
non-reversible case via the Hastings ratio:
\begin{equation}
    \label{eq:non_rever_ratio}
    R_{\Gamma}(u,v) := \begin{cases}
        \frac{\Gamma(u,v) + \pi(v)Q(v,u)}{\pi(u)Q(u,v)} & \pi(u)Q(u,v) \neq 0\\
        1 & \text{otherwise}
    \end{cases}
\end{equation}
where $\Gamma(u,v) := \pi(u)P(u \rightarrow v) - \pi(v)P(v \rightarrow u)$, subject to 
$-\pi(v)Q(v,u) \le \Gamma(u,v) \le \pi(u)Q(u,v)$ to ensure non-negativity.

\label{subsec:supervised_pagerank}

\textbf{Supervised PageRank.} PageRank~\citep{page1998pagerank} models web navigation as a Markov chain, ranking pages 
via the stationary distribution $\pi(v) = \alpha(v) + c\sum_{\{u|(u,v)\in E\}} 
\frac{\pi(u)}{|\{w|(u,w)\in E\}|}$, where $\alpha$ ensures irreducibility. 
\citet{zhukovisky2014pagerank} extends this to \textit{Supervised PageRank}, learning 
personalized transition probabilities via functions $f: X \times V \rightarrow \mathbb{R}$ 
and $g: Y \times E \rightarrow \mathbb{R}$. We build on the gradient-free variant of 
this framework~\citep{NIPS2016_1f34004e}.

%% file: sections/boil.tex
\section{Blackbox Oracle Information Learning (BOIL)}
\label{sec:boil}
In this section, we will begin by formulating the problem environment. Subsequently, we will demonstrate how information is extracted from the environment structure to address the coverage problem utilizing movement constraints (Flow constraints) enforced by the environment.  In the coverage task, we consider a node is covered if it comes in visibility of any agent, and the agents want to cover all nodes equally as many times as possible.\footnote{In this work, we make a distinction between patrolling the area and patrolling specific areas in the environment. We call the first one coverage and the latter patrolling.} Note that throughout this work, we assume homogeneity among agents, meaning they possess identical physical capabilities unless stated otherwise. Proofs are deferred to the Appendix.

    \subsection{Coverage Problem Formulation}
    \label{subsec:coverage_formulation}
    
    Similar to prior research, we represent the environment with a graph. However, a key distinction lies in our utilization of an undirected graph $\mathcal{G}(V, E)$ to encode the topography or environment semantics alongside a directed graph $\mathcal{G}_d(V_d, E_d)$ to model the potential movement space of the robots/agents. Such a formulation allows a rich representation for real world settings like cases when agents can go down a hill but cannot come up the same way because of physical constraints or model a trapdoor like situation that arise in case of fire exits and one-ways in urban settings. It's important to note that all vertices $v \in V_d$ possess self-loops, denoted by $(v, v) \in E_d\,\, \forall \, v\in V_d$, indicating that $\mathcal{G}_d$ is not a simple directed graph.\footnote{Directed graphs without self-loops are referred to as simple directed graphs.} Additionally, we assume that $\mathcal{G}_d$ is strongly connected. This assumption is made to prevent scenarios where the agent reaches states from which it cannot return to its initial state or where certain parts of the graph are unreachable regardless of the number of steps taken. Furthermore, we assume that the traversal time for any edge in $E_d$ is independent of the agent's previous actions.
    
    \citet{stern2006multiagent}~utilizes a visibility matrix as a map listing nodes visible from a particular node. This matrix formally defines a fitness function for the genetic algorithm in prior research. Here, we define visibility as a function $V_s: E_d \rightarrow [0,1]^{|V|}$, where $V_s((u,v))(w)$ represents the value for node $w \in V$ when $V_s$ is assessed on $(u,v) \in E_d$ which allows us to define unidirectional vision for agents. This value gives the probability that node $w$ is visible when an agent crosses edge $(u,v)$. 

    \begin{definition}
        Let $V_s((\Bar{u}, \dots, \Bar{v}))$ denote the visibility for a path $(\Bar{u}, w_1, w_2, \dots, w_{l-1}, \Bar{v})$, where $T_{\Bar{u}\Bar{v}}$ denotes the time required for the agent to traverse the path. Let $\mathcal{V}_w(t): [0, T_{\Bar{u}\Bar{v}}] \rightarrow \{0,1\}$ be an indicator function, which is 1 if the node $w$ is visible at time $t$. Then,
            $$\forall w \in V, \,\, V_s((\Bar{u}, \dots, \Bar{v}))(w) = \frac{1}{T_{\Bar{u}\Bar{v}}}\int_0^{T_{\Bar{u}\Bar{v}}} \mathcal{V}_w(t) dt $$
    \end{definition}



    Consider a system with $n$ agents. Let the function $h_{(u,v)}^i(t)$ be $1$ for all $t$ when agent $i$ is crossing the edge $(u,v)$, and otherwise $0$. These functions encode the \textit{optimal} trajectory information for the particular environment. We postulate a theoretical oracle $\mathcal{O}$ that can generate these functions. Crucially, $\mathcal{O}$ serves solely as an existence proof for the trajectory information. We show that it is possible to obtain information about these trajectories without actually querying the oracle. Suppose we make a vector $Y_t$, which represents the probability that a node in $V$ is visible at time $t$ from any agent. Then we can write $Y_t$ as follows:

    \begin{equation}
        \label{eq:visibility_time}
        Y_t = \vec{1}-\prod_{i=1}^n \left(\vec{1}-\sum_{(u,v) \in E_d}h_{(u,v)}^i(t)V_s((u,v))\right)
    \end{equation}

    where the product of vectors represents the element-wise multiplication. For the rest of the work, consider that a direct vector of products represents the element-wise product unless otherwise mentioned. We can define the node visibility probability for the time duration $[0, T]$ as $P_V(w) = \left(\int_0^T Y_t(w)dt\right)/T$. The coverage problem can now be defined as maximizing the common information~\cite{liu2010common} for random variables that represent the node visibility. We use Theorem \ref{theorem:probability_distribution} given below to find a bound on the integral and reduce the problem to minimizing the loss function $\mathcal{L}$ as follows:
    \begin{equation}
        \label{eq:loss_function_visibility}
        \begin{split}
            \mathcal{L} = \sum_{w \in V} -\mathcal{A}(w)\log\mathcal{A}(w)\,\,\text{where}\,
            \mathcal{A}(w) = \sum_{(u,v)\in E_d} P((u,v))V_s((u,v))(w)            
        \end{split}
    \end{equation} 
    The details of the reduction are deferred to the appendix along with proof of Theorem~\ref{theorem:probability_distribution}. The proof also covers why the loss is independent of the agent count.

    \begin{theorem}
        \label{theorem:probability_distribution}
        Suppose we define,
        $$P((u,v)) = \frac{1}{nT}\sum_{i=1}^n\int_0^T h_{(u,v)}^i(t) dt \,\text{ and } \, \mathcal{X}^i(t) = \sum_{(u,v) \in E_d}V_s((u,v))(w)h^i_{(u,v)}(t)$$
        then $P$ is a probability distribution on the edges $E_d$ and the following holds $\forall w \in V$:
        \begin{equation}
            \begin{split}
                nT\sum_{(u,v) \in E_d} P((u,v))V_s((u,v))(w) \ge \int_0^T Y_t(w) dt \ge
                T\sum_{(u,v) \in E_d} P((u,v))V_s((u,v))(w)\\
                  nT\sum_{(u,v) \in E_d}P((u,v))V_s((u,v))(w)
                  \le
                  \sum_{
                      1\le i<j \le n
                  } \int_0^T \mathcal{X}^i(t)\mathcal{X}^j(t)dt + \int_0^T Y_t(w) dt      
            \end{split}
        \end{equation}
        
    \end{theorem}
    Theorem \ref{theorem:probability_distribution} shows that minimizing Equation \ref{eq:loss_function_visibility} increases a lower bound on the common information between environment states and agent visitation events. The bound is not intended to be tight in a constant-factor sense. Instead, it establishes that improvements in the optimization objective correspond monotonically to increased shared information, thus serving as a principled surrogate objective. While tighter bounds could be achieved via more complex spectral analysis, our priority is maintaining linear $O(|E|)$ scalability. The current bound suffices to drive the gradient-free optimization towards improved coverage without incurring cubic computational costs associated with exact eigen-decomposition.


    
    \subsection{Solving as Flow Constraint Problem}
    \label{subsec:flow_constraint}
    Any agent can only cross the edge $(u,v) \in E_d$ only if the agent is on the node $u \in V_d$. Moreover, the agent can stay at a node $u$ only if it has reached node $u$. All the agents in the system are constrained by the structure of the directed graph $\mathcal{G}_d$. Following the above argument, we can decompose $P((u,v))$ into two probabilities change (shown in Theorem \ref{theorem:node_probability}): 1) Probability that the agent is on node $u$ represented as $\pi(u)$ and 2) probability that the agent moves along the edge $(u,v)$ given that the agent is on the node $u$ represented as $P(u \rightarrow v)$. Hence,

    \begin{equation}
        \label{eq:probability_decomposition}
        \begin{split}
            \implies &P((u,v)) = \pi(u)P(u \rightarrow v) \,\,\text{ where }\, \sum_{\{v|(u,v)\in E_d\}} P(u \rightarrow v) = 1            
        \end{split}
    \end{equation}


    \begin{theorem}        
        \label{theorem:node_probability}
        The function $\pi:V_d \rightarrow \mathbb{R}$ is a probability distribution on the nodes in $V_d$. 
    \end{theorem}
 
    Furthermore, the net flow of the probability should be balanced across the graph, implying that $P(u \rightarrow v)$ for all $(u,v) \in E_d$ and $\pi(u)$ for all $u \in V_d$ have values such that the global balanced condition (Equation \ref{eq:global_balanced_condition}) is satisfied. Hence, the movement can be modeled as a Markov chain with the nodes $V_d$ as the state space. However, we do not know the transition probabilities $P(u \rightarrow v)$.

    \begin{equation}
        \label{eq:global_balanced_condition}
        \sum_{u \in V_d} \pi(u)P(u\!\rightarrow\! v) = \pi(v)
    \end{equation}

    Previous works have studied~\cite{mathew2011metrics, miller2015ergodic, abraham2018decentralized} how multiple agents can approximate a given distribution defined similarly to the definition of $P((u,v))$ as given in theorem \ref{theorem:probability_distribution}. Please observe that the loss function \ref{eq:loss_function_visibility} only depends on the transition probabilities $P((u,v))$, implying that we need to only find the correct values for $P(u \rightarrow v)$ that minimizes the loss function. Hence, we model the problem as a Supervised PageRank (\cref{subsec:supervised_pagerank}) optimization problem, and get Algorithm \ref{alg:boil} by building on \citet{NIPS2016_1f34004e}.

    Note how we have only constrained the flow of probability, as Pagerank only ensures the global balanced condition (\cref{eq:global_balanced_condition}). The oracle is supposed to provide continuous paths, but we solved for a softer probabilistic constraint. Hence, running Algorithm \ref{alg:boil} allows us to access some information from the oracle. While previous work does not require an oracle, it uses hard constraints, thereby restricting custom control over the design trade-offs. Using the oracle formulation allows us to perform \textit{fine-grained estimation}. The parameters are log-probabilities for stability.

    \begin{algorithm}[h]
      \caption{BOIL Algorithm}
      \label{alg:boil}
      \begin{algorithmic}
        \STATE {\bfseries Input:} step size $\mu > 0$, number of steps $N \ge 1$, initial transition vector $\vec{p}_0$ which represents $\ln P(u \rightarrow v)\,\, \forall (u,v) \in E_d$, directed graph $\mathcal{G}_d$, loss function $\mathcal{L}$
        \STATE {\bfseries Output:} optimized transition vector $\vec{p}_k$ and PageRank vector $\vec{x}_k$
    
        \STATE $k \gets 0$, $m \gets |\vec{p}_0|$, $\vec{x}_0 \gets \mathrm{PAGERANK}(\vec{p}_0, \mathcal{G}_d)$   \COMMENT{$|\cdot|$ gives the dimension of the vector.}  
        \WHILE{$k < N$}
            \STATE $\vec{r} \gets$ random vector on the $m$-dimensional unit sphere
            \STATE $\vec{q} \gets$ normalized $(\vec{p}_k + \vec{r})$ such that $\sum_{\{v \mid (u,v) \in E_d\}} P(u \rightarrow v) = 1$
            \STATE $\vec{y} \gets \mathrm{PAGERANK}(\vec{q}, \mathcal{G}_d)$

            \STATE $\vec{g} \gets m \big( \mathcal{L}(\vec{q}, \vec{y}) - \mathcal{L}(\vec{p}_k, \vec{x}_k) \big)\vec{r}$
            \STATE $\vec{p}_{k+1} \gets$ normalize $(\vec{p}_k - \mu \vec{g})$ such that $\sum_{\{v \mid (u,v) \in E_d\}} P(u \rightarrow v) = 1$
            \STATE $\vec{x}_{k+1} \gets \mathrm{PAGERANK}(\vec{p}_{k+1}, \mathcal{G}_d)$
            \STATE $k \gets k + 1$
        \ENDWHILE
    
        \STATE $k^\star \gets \arg\min_{k \in \{0,\dots,N-1\}} \mathcal{L}(\vec{p}_k, \vec{x}_k)$
        \STATE {\bfseries Return:} $\vec{p}_{k^\star}, \vec{x}_{k^\star}$
      \end{algorithmic}
    \end{algorithm}

    The complexity of pagerank is $O(|E|)$, making the overall complexity $O(NK|E| + N \Omega)$, where $N$ is the number of iterations for gradient free optimization; $\Omega = \sum_{(u,v,) \in E_d} \sum_{w \in V_d} \delta_{V((u,v))(w) > 0}$ and $K$ the maximum iterations for pagerank. $\Omega$ measures how many calculations contribute to the loss function.
            
    \subsection{Fine Grained Estimation}
    \label{subsec:fine_grained_estimation}
    The estimator that we used in the coverage problem employs a distribution over the node space. However, the parameters required to solve for the node space might be suboptimal for the available compute. Furthermore, increasing the resolution of the graph in the node space can lead to an increase in the parameter space larger than the available compute. Moreover, it might be favorable to obtain additional information about the movement of the agents along the time axis in place of spatial information. Observe how the presented technique only requires a strongly connected state space, and the visibility function depends only on the state and not on the time the state is reached. Now, we show various methods to increase the state space, which gives more information about the system with a relatively low increase in the parameter space.

    Consider the coverage problem. Suppose we are given a continuous path $(\Bar{u}, w_1, w_2, \dots, w_{l-1}, \Bar{v})$ and we want to find out the probability the agent should take this path to minimize the loss.

    \begin{theorem}
    \label{theorem:close_loop}
        Given a path $(\Bar{u}, w_1, w_2, \dots, w_{l-1}, \Bar{v})$, the edge set $E_d$ can be modified to add an additional edge between the node $\Bar{u}$ and $\Bar{v}$ represented as $(\Bar{u}, \dots, \Bar{v})$. Optimizing the loss function $\mathcal{L}$ on the original edge set $E_d$ is the same as optimizing the loss function over the modified edge set. Furthermore, the probability that the agent takes the path is the probability for the edge $(\Bar{u}, \dots, \Bar{v})$ found by optimizing the loss on the modified edge set.
    \end{theorem}

    Observe that Theorem \ref{theorem:close_loop} can be applied repeatedly, and it increases only $1$ parameter per path. Note how the base formulation can be obtained by repeatedly applying Theorem \ref{theorem:close_loop} on the empty edge set with paths of unit length. Adding higher length paths puts a hard continuity constraint in addition to the soft probabilistic one thus giving a fine-grained control. Suppose we can handle twice the number of parameters and want to get more information about the movement of the agents in the time domain, then we can use Theorem \ref{theorem:split_flow}. Let $\mathcal{P}(p)$ represent the set of all countable partitions of the interval $[0, T]$ for $p \in (0,1)$ such that:

    \begin{equation}
        \begin{split}
            \int_0^T I(t)\, dt = pT, \,\, \forall \, (t_1, t_2, \dots) \in \mathcal{P}(p) \,\,\text{ where }\,
            I(t) =\begin{cases}
                1 & t \in t_i, \text{i is even}\\
                0 & \text{otherwise}
            \end{cases}            
        \end{split}
    \end{equation}
    
    For any partition $(t_1, t_2, \dots) \in \mathcal{P}(p)$, we can split the $h^i_{(u,v)}(t)$'s into two functions $\Hat{h}^i_{(u,v)}(t)$ and $\Bar{h}^i_{(u,v)}(t)$ such that:

    \begin{equation}
    \label{eq:split_functions}
        \begin{split}
            \hat{h}^i_{(u,v)}(t) = \begin{cases}
                h^i_{(u,v)}(t) & t \in t_i, \text{i is odd}\\
                0 & \text{otherwise}
            \end{cases}\,\text{ and }\,
            \bar{h}^i_{(u,v)}(t) = \begin{cases}
                h^i_{(u,v)}(t) & t \in t_i, \text{i is even}\\
                0 & \text{otherwise}
            \end{cases}            
        \end{split}
    \end{equation}

    Observe that $h^i_{(u,v)}(t) = \Hat{h}^i_{(u,v)}(t) + \Bar{h}^i_{(u,v)}(t)$.
    \begin{theorem}
    \label{theorem:split_flow}
    For any partition $(t_1, t_2, \dots) \in \mathcal{P}(p)$, let
        \begin{equation*}
            \begin{split}
                \hat{P}((u,v)) = \frac{1}{npT}\sum_{i=1}^n\int_0^T \hat{h}_{(u,v)}^i(t) dt  \,\text{ and }\,
                \bar{P}((u,v)) = \frac{1}{n(1-p)T}\sum_{i=1}^n\int_0^T \bar{h}_{(u,v)}^i(t) dt                
            \end{split}
        \end{equation*}
        
    Then, $\Hat{P}$ and $\Bar{P}$ form a probability distribution over the edges $E_d$. Furthermore, when decomposed in edge transition probabilities, represented as $\Hat{P}(u \rightarrow v)$ and $\Bar{P}(u \rightarrow v)$ for edge $(u,v) \in E_d$, and a distribution over the nodes, represented as $\Bar{\pi}(u)$ and $\Hat{\pi}(u)$ for vertex $u \in V_d$, respectively, then the following holds:
        \begin{equation*}
            \pi(u) = p\Hat{\pi}(u) + (1-p)\Bar{\pi}(u)\\
        \end{equation*}
    Furthermore, the following constraints are sufficient to ensure flow constraints:
    \begin{equation*}
        \begin{split}
            \sum_{\{v|(u,v) \in E_d\}} \Hat{\pi}(u)\Hat{P}(u \rightarrow v) = \Hat{\pi}(v) \,\text{ and }\,
            \sum_{\{v|(u,v) \in E_d\}} \Bar{\pi}(u)\Bar{P}(u \rightarrow v) = \Bar{\pi}(v)                        
        \end{split}
    \end{equation*}
    
    \end{theorem}

    To ensure that the found distributions are as different as possible, we can optimize for the following loss instead of the original loss $\mathcal{L}$:
        $$\mathcal{L}_{\Lambda} = \mathcal{L} - \frac{1}{2}\sum_{(u,v) \in E_d} \Lambda_{(u,v)}\left(\Hat{P}((u,v)) - \Bar{P}((u,v))\right)^2 \,\,\text{where $\Lambda: E_d \rightarrow \mathbb{R}$ is some fixed function}$$

    As the theorem holds for any partition, optimizing the loss will give us the distributions that minimize the loss over all partitions in $\mathcal{P}$. 

%% file: sections/results.tex
\section{Experiments \& Results}
\label{sec:results}

We evaluate BOIL on two distinct topographic environments: a ``Small" ($36 \times 36$) and a ``Large" environment (deferred to Appendix~\ref{app:large_environment}), both featuring complex occlusions. All computation was performed using 19 CPU cores without GPU acceleration. We explicitly exclude Deep Multi-Agent Reinforcement Learning (MARL) baselines (e.g., MAPPO~\cite{yu2022surprising}, MASAC~\cite{lowe2017multi}) from this study; while powerful, these methods typically require days of GPU-based training to converge on such agent history dependent tasks. Another reason why we exclude MARL methods is because the solution depends highly on the exact path taken by the agents. It is well known that MARL methods are more suitable for dynamic or adversarial scenarios rather than path finding problems over a long horizon. Furthermore, there is an issue of reward shaping for the task. Given the large variety of approaches and modern advances in deep MARL methods, we want to avoid comparisons that are misrepresenting the pros or cons of our method in contrast to MARL methods in general. Moreover, as our method is more about the process of finding the distribution and not the exact path, the comparison is not fair for either. MARL is trying to solve the harder problem of finding exact paths while ours is primarily about finding some information about the paths. Given the computational cost for MARL methods, it renders them impractical for our target application: rapid, ad-hoc deployment in novel facilities where policies must be generated on-site using standard onboard hardware. In contrast, BOIL derives effective strategies within a few hours on CPU. Consequently, we compare against heuristic and graph-theoretic baselines that operate within a similarly accessible computational regime.

\subsection{Agent Strategies}
\label{subsec:agent_strategies}

Table~\ref{tab:agent_strategies} summarizes the agent strategies evaluated in our experiments. 
The Frontier and Sample agents use a count vector $C$, where $C(w)$ is the number of times 
node $w$ was visible, to define transition probabilities:
\begin{equation}
    P_{\text{frontier}}(u \to v) \propto \frac{1}{|V((u,v))|}
    \sum_{w \in V((u,v))} \frac{1}{\max(C(w), 10^{-6})}
\end{equation}
The Sample agent augments this with a Q-function using BOIL's learned distribution via 
MH sampling (Section~\ref{subsec:nonreversible_markov_chain}):
\begin{equation}
    Q(u,v) = 1 + \frac{\lambda}{|V((u,v))|}
    \sum_{w \in V((u,v))} \frac{1}{\max(C(w), 10^{-6})}
\end{equation}
with $\lambda = 10$ for all experiments controlling the preference for frontier exploration.

\begin{table}[h]
\centering
\resizebox{\columnwidth}{!}{%
\begin{tabular}{ll}
\toprule
\textbf{Agent} & \textbf{Description} \\
\midrule
Random         & Uniform random walk over adjacent edges \\
OptRandom      & Samples any edge in $E_d$ uniformly; may teleport. Strong unconstrained baseline \\
Frontier       & Prioritizes low-visit nodes via count vector $C$; randomized to avoid local optima \\
Comm Frontier  & Frontier with shared $C$ across all agents \\
Sample         & MH sampling from BOIL distribution with frontier Q-function bias \\
Comm Sample    & Sample agent with shared $C$ across all agents \\
Optimal        & Direct MH sampling from BOIL dist.; no continuity constraint. Upper bound reference \\
\bottomrule
\end{tabular}%
}
\caption{Summary of agent strategies}
\label{tab:agent_strategies}
\end{table}

\subsection{Experiments}

The small environment comprises a $36\times 36$ field, featuring tall walls and uneven topography illustrated in Figure \ref{fig:small_environment}. In the depiction, red blocks denote tall walls, impassable and invisible to the agent. Elevation levels are represented by varying shades of blue, with lighter shades indicating lower elevation and darker shades signifying higher elevation. We run simulation with homogeneous teams with $8$ agents with unidirectional visibility. Each agent can only see a point $3.5$ units away at maximum, implying each agent can only observe $\approx 3\%$ of the area at maximum at any given instant.

\begin{figure}[h]
    \centering
    \begin{subfigure}[t]{0.25\linewidth}
        {\includegraphics[scale=0.2]{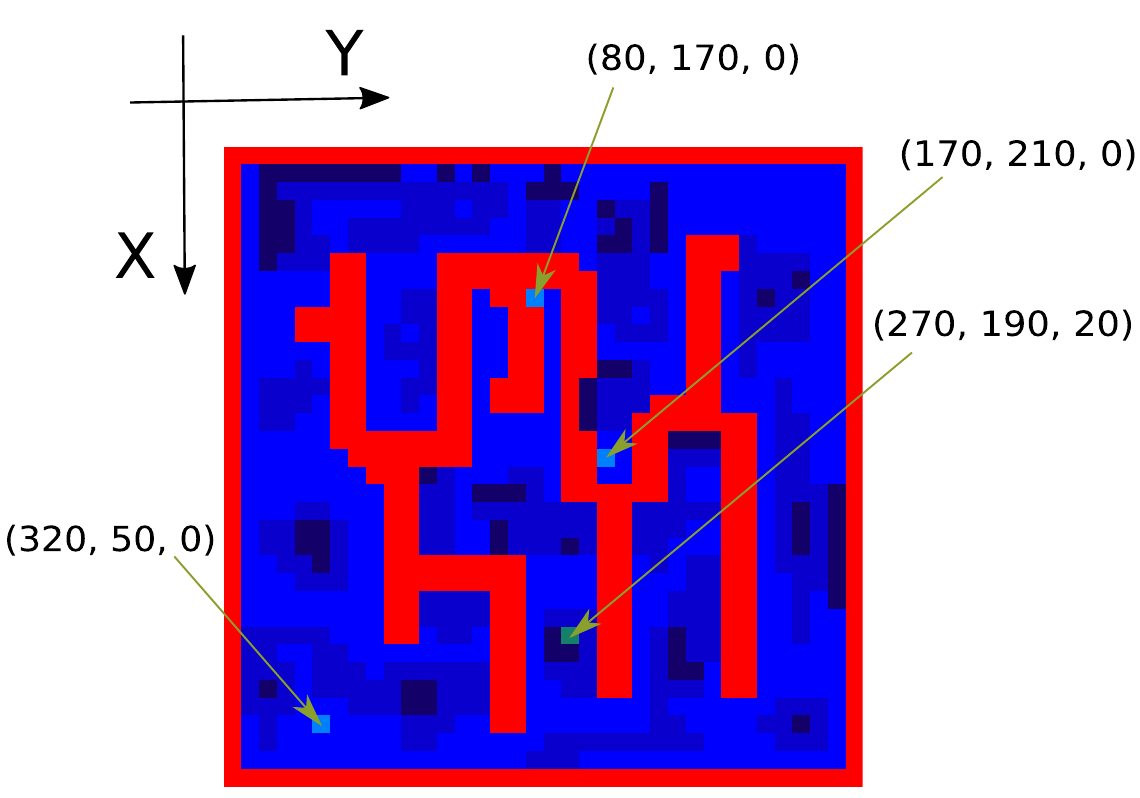}}
        {\caption{\small Small Environment}\label{fig:small_environment}}        
    \end{subfigure}
    \hspace{0.03\linewidth}
    \begin{subfigure}[t]{0.3\linewidth}
        {\includegraphics[scale=0.13]{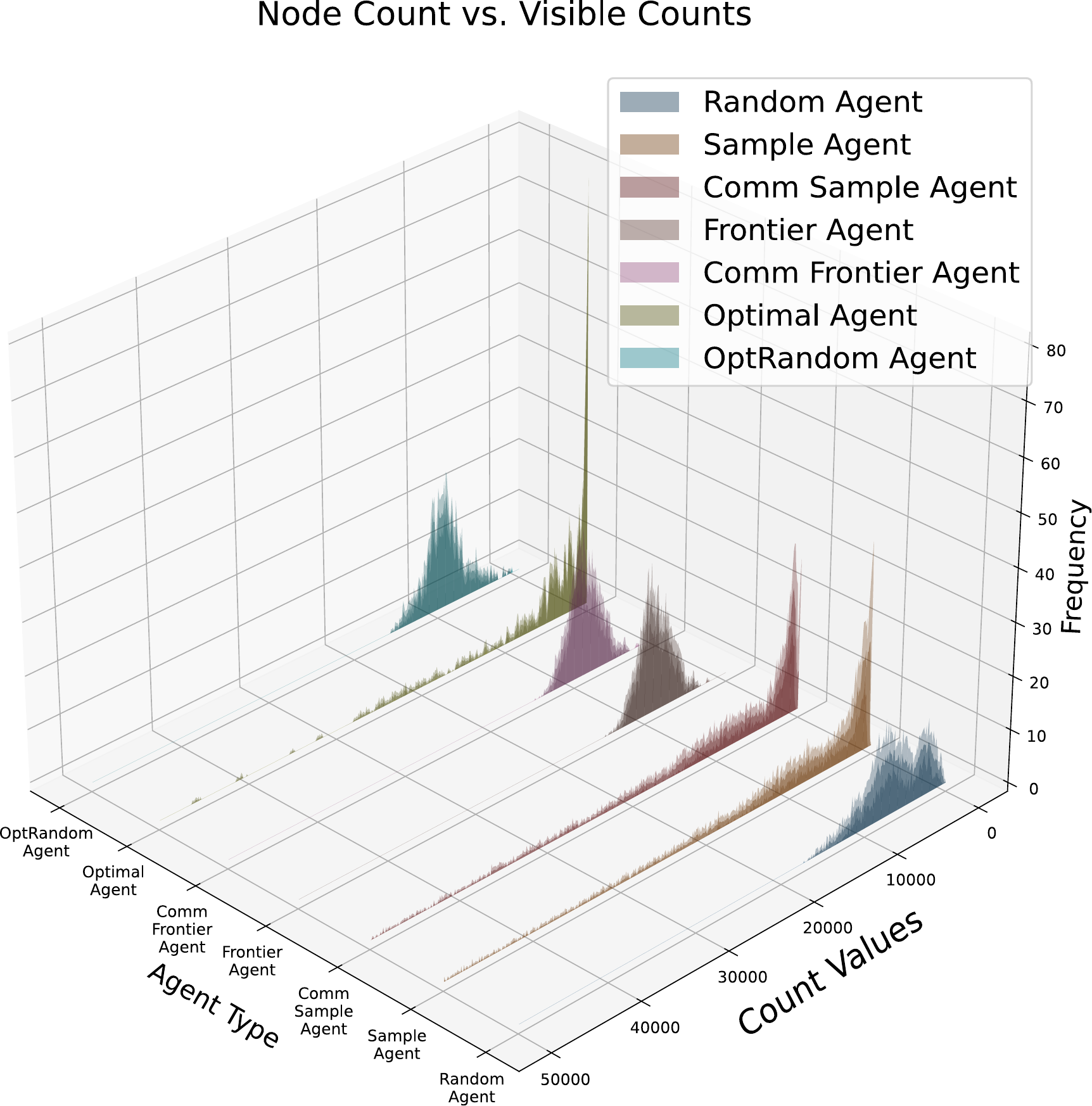}}
        \caption{\small Frequency histograms}
        \label{fig:small_environment_vis_counts}
    \end{subfigure}
    \hspace{0.03\linewidth}
    \begin{subfigure}[t]{0.35\linewidth}
        {\includegraphics[scale=0.15]{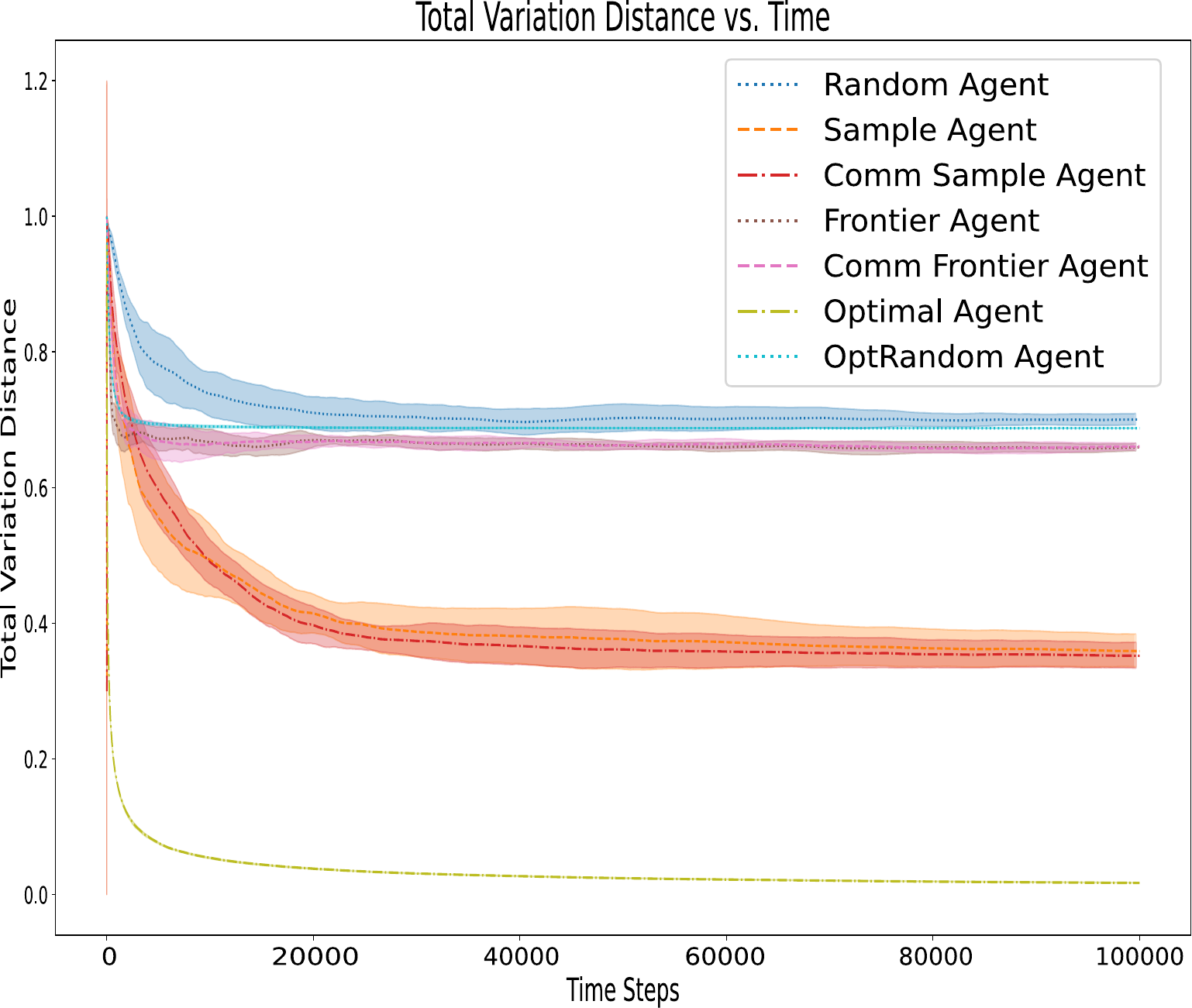}}
        {\caption{\small Variational distance}\label{fig:small_environment_convergence}}
    \end{subfigure}
    \caption{\small Figure (a) depicts the small environment with a color encoding easier for understanding the topology. The \textit{cyan-like} colors are pixels that encode the elevation, and that extra data is supposed to be stored for nodes representing those pixels. Figure (b): The Z-axis illustrates the counts of visible nodes corresponding to values on the Y-axis, while the X-axis denotes the types of homogeneous agents considered in the experiment. Curves' opacity signifies variations in values, with lower opacity indicating the upper end of the variance and higher opacity representing the lower end for frequencies. The central divider indicates the mean frequencies. Figure (c) depicts Total variation distance between the approximate distribution of the agent trajectories and the distribution found using the BOIL process for the small environment for all the time-steps. \textbf{BOIL Distribution Strategies:} \textit{Sample Agent, Comm Sample Agent, Optimal Agent}}
\end{figure}

The agent's movement is constrained, permitted only from lighter to darker shades in sequential order. However, it can descend directly from a darker shade to any blue-colored pixel. Elevational differences significantly impact visibility, with higher elevations offering greater visibility and vice versa. The \textit{cyan-like} colored pixels convey an elevation in the same three levels just that we use these points specifically to understand the behaviors of the different agents. We conducted simulations ten times, each spanning $10^5$ steps, to explore the effect of randomness. Figure \ref{fig:small_environment_vis_counts} presents quantitative visibility counts, while Figure \ref{fig:small_environment_convergence} displays the total variation distance (Definition~\ref{def:total_variation_distance}) of empirical distributions to those obtained through the BOIL process.

\begin{definition}
    \label{def:total_variation_distance}
    Given two probability $\mu$ and $\nu$ on the discrete state space $S$,
        $$d(\mu, \nu) = \frac{1}{2}\sum_{x \in S} |\mu(x) - \nu(x)|$$
    where $d$ is the total variation distance~\cite{gibbs2002choosing}.
\end{definition}

\begin{figure*}[b]
    \centering
    \includegraphics[scale=0.15]{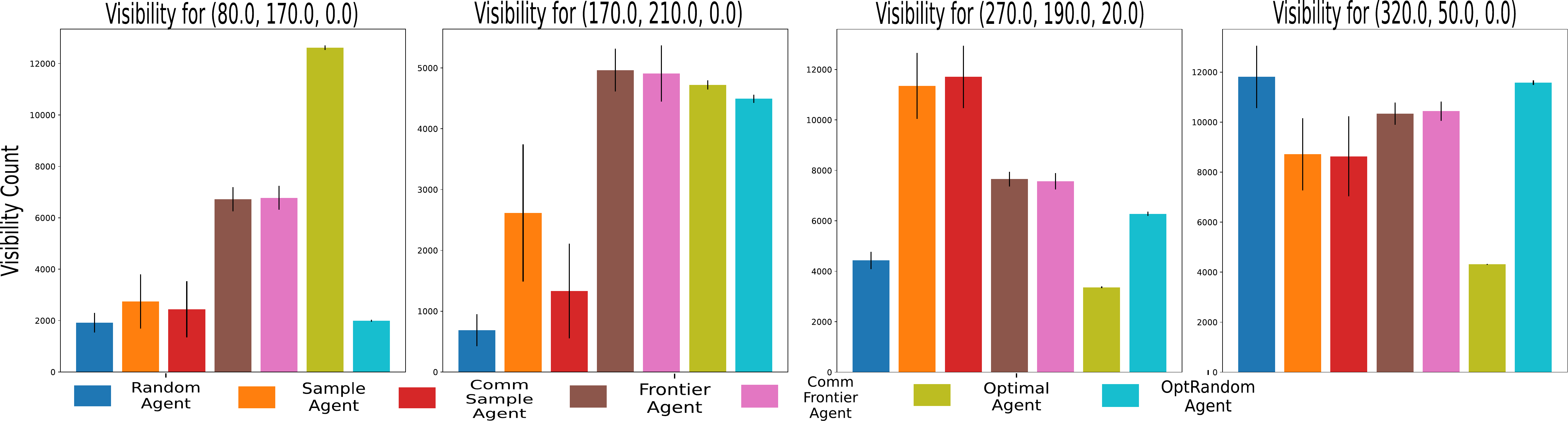}
    \caption{\small The figure shows the visibility counts for the \textit{cyan-like} colored pixels in figure \ref{fig:small_environment} for the different agent strategies. The bars shows the mean and the errorbars show the variance.}
    \label{fig:small_environment_tracking}
\end{figure*}

In the context of the uniform coverage problem, the objective is to ensure that agents explore every area equally. Figure \ref{fig:small_environment_vis_counts} reveals that Optimal Agents exhibit a prominent peak, indicating effective utilization of visibility across different locations. Interestingly, all strategies except for Sample and Comm Sample demonstrate distributions similar to OptRandom , suggesting suboptimal coverage patterns.

Despite their similarities to Optimal agents, Sample and Comm Sample strategies fail to achieve the optimal distribution in $10^5$ steps, as evident from Figure \ref{fig:small_environment_convergence}. While the distance of the Optimal agent converges to zero, it plateaus for the Sample and Comm Sample agents, though it continues to decrease slowly. Furthermore, the figure illustrates significantly higher and similar distances for all other strategies. These observations highlight the limitations of frontier exploration techniques, particularly in complex environments and over extended time horizons. Although Sample-based strategies exhibit similar distributions to Optimal agents, they fail to achieve the desired optimal distribution.

Figure \ref{fig:small_environment_tracking} provides a cumulative count of how many times a particular point was visible. It is intriguing to observe that the Optimal agents exhibit a pronounced preference for observing the corner point significantly more times than others. Conversely, both Frontier and Comm Frontier agents exhibit higher counts than Sample and Comm Sample agents in $3$ out of $4$ points. Notably, the point with a higher elevation, visible more frequently from Sample and Comm Sample agents, is surrounded by areas at higher elevations. This observation suggests that the simple sampling-based strategy effectively leverages the distribution obtained from BOIL to prioritize high-visibility points. This finding underscores the adaptability and efficacy of sampling-based strategies in leveraging information extracted through the BOIL process to enhance visibility and coverage in complex environments. Further analysis can be found in Appendix \ref{app:further_analysis}.

%% file: sections/discussion.tex
\section{Discussion}
\label{sec:discussion}

    \subsection{Extension to Patrolling and Reachability}
    \label{subsec:extension_patrolling_reachability}
    The technique can be extended to other problems like patrolling and reachability problems. Patrolling can be defined as the problem where the agents want to patrol certain points on the graph at all times. Reachability can be defined as the problem where the agent wants to ensure that it can reach any particular place within a certain amount of time.

    Consider a set of vertices $V_p \subseteq V$ that needs to be patrolled by the agents for the patrolling problem. Keeping the same setting as in the case of coverage problem discussed before, we only need to modify the loss function to the following:
    \begin{equation}
        \label{eq:loss_function_patrolling}
        \begin{split}
            \mathcal{L}_p = \sum_{w \in V_p} -\mathcal{A}(w)\log\mathcal{A}(w) \,\,\text{ where }\,
            \mathcal{A}(w) = \sum_{(u,v)\in E_d} P((u,v))V_s((u,v))(w)            
        \end{split}
    \end{equation}
    The modified loss only minimizes the expected information from points of interest and discards the information from other nodes. 
    For the reachability problem, in place of a visibility function, we define a reachability function $R: E_d \rightarrow (0,1)^{|V|}$ where $R((u,v))(w)$ represents the probability that an agent can reach the node $w \in V$ if the agent is traversing the edge $(u,v)$ within a predefined time $T_R(w)$. We can now replace the visibility function in the formulation for the coverage problem.

    \subsection{Conclusion, Limitations \& Future Work}
    \label{subsec:conclusions}
    
    In this paper, we introduced the Blackbox Oracle Information Learning (BOIL) process as a scalable solution for extracting valuable insights from the environment structure in multi-agent systems. Leveraging the Pagerank algorithm and information theory, BOIL enables the extraction of information about long-term agent behavior. We demonstrated the flexibility of the formulation by applying it to various problems such as coverage, patrolling, and stochastic reachability, all converted into common information maximization problems solvable using the same technique.
    
    
    An important assumption of our work is the availability of reliable information about the environment. With the entire process being offline, it is not possible to deal with dynamic changes in the environment. To address this limitation, future work aims to extend the framework to allow for online updates in the BOIL process, coupled with a controller capable of independently utilizing extracted information. This direction holds promise, as BOIL is an iterative process seemingly independent of noise in control policies.
    

%% file: appendix/probability_sandwich.tex
\section{BOIL}
\begin{lemma}
\label{lemma:visibility_linearity}
Given two continuous paths $(\Bar{u}, x_1, \dots, x_{l-1}, \Bar{v})$ and $(\Bar{v}, y_1, \dots, y_{m-1}, \Bar{w})$ with traversal times represented as $T_{\Bar{u}\Bar{v}}$ and $T_{\Bar{v}\Bar{w}}$ respectively, for all $z \in V$:

    \begin{equation}
    \begin{aligned}
        V_s((\Bar{u},\dots, \Bar{v}, \dots, \Bar{w}))(z) &= \frac{T_{\Bar{u}\Bar{v}}V_s((\Bar{u}, \dots, \Bar{v}))(z) + T_{\Bar{v}\Bar{w}}V_s((\Bar{v}, \dots, \Bar{w}))(z)}{T_{\Bar{u}\Bar{v}} + T_{\Bar{v}\Bar{w}}}            
    \end{aligned}
    \end{equation}        
\end{lemma}
\begin{proof}
    \label{proof:lemma_visibility_linearity}
    Let the travel time for the path $(\Bar{u}, \dots, \Bar{v}, \dots, \Bar{w})$ be $T_{\Bar{u}\Bar{w}}$. It is easy to observe that $T_{\Bar{u}\Bar{w}} = T_{\Bar{u}\Bar{v}} + T_{\Bar{v}\Bar{w}}$ as we assume that the traversal time is independent of previous actions of the agent.

    \begin{equation}
            T_{\Bar{v}\Bar{w}}V_s((\Bar{v},\dots, \Bar{w}))(z) = \int_0^{T_{\Bar{v}\Bar{w}}} \mathcal{V}_z(t) dt
            = \int_{T_{\Bar{u}\Bar{v}}}^{T_{\Bar{u}\Bar{w}}} \mathcal{V}_z(t - T_{\Bar{u}\Bar{v}}) d(t - T_{\Bar{u}\Bar{v}} = \int_{T_{\Bar{u}\Bar{v}}}^{T_{\Bar{u}\Bar{w}}} \mathcal{V}_z(t) d(t)
    \end{equation}

    Hence, we have that,
    \begin{equation}
    \begin{split}
        T_{\Bar{u}\Bar{v}}V_s((\Bar{u},\dots, \Bar{v}))(z) + T_{\Bar{v}\Bar{w}}V_s((\Bar{v},\dots, \Bar{w}))(z)
        &= \int_0^{T_{\Bar{u}\Bar{v}}} \mathcal{V}_z(t) d(t) + \int_{T_{\Bar{u}\Bar{v}}}^{T_{\Bar{u}\Bar{w}}} \mathcal{V}_z(t) d(t)\\
        &= \int_0^{T_{\Bar{u}\Bar{w}}} \mathcal{V}_z(t) d(t)
    \end{split}
    \end{equation}
\end{proof}

The following is proof for Theorem \ref{theorem:node_probability}
\begin{proof}
    \label{proof:node_probability}
    Using the fact that we have probability distribution on the edges $E_d$, we have,
    \begin{equation*}
    \begin{split}
        \sum_{(u,v) \in E_d} P((u,v)) &= \sum_{(u,v) \in E_d} \pi(u)P(u \rightarrow v)\\
        &= \sum_{u \in V_d} \sum_{\{v| (u,v) \in E_d\}} \pi(u)P(u \rightarrow v)\\
        &= \sum_{u \in V_d} \pi(u) \sum_{\{v| (u,v) \in E_d\}} \pi(u)P(u \rightarrow v)\\
        &= \sum_{u \in V_d} \pi(u)
    \end{split}
    \end{equation*}
\end{proof}

\label{app:probability_sandwich}
\begin{theorem}
\label{theorem:sandwich}
    Let us define the following for any given natural $n$, time $0 \le t \le T$, and $w \in V$:
    \begin{equation*}
        \mathcal{X}^i = \sum_{(u,v) \in E_d}h^i_{(u,v)}(t)V_s((u,v))(w)
    \end{equation*}
    Then the following holds:
    \begin{equation*}
        \sum_{i=1}^n \mathcal{X}^i - \sum_{1\le i<j \le n} \mathcal{X}^i\mathcal{X}^j \le Y_t(w) \le \sum_{i=1}^n \mathcal{X}^i
    \end{equation*}
\end{theorem}
\begin{proof}
    By definition of $h^i_{(u,v)}(t)$, only one of them can be $1$ at a given time $t$ for any $1 \le i \le n$. That is, $\sum_{(u,v) \in E_d} h^i_{(u,v)}(t) = 1$

    As $0 \le V_s((u,v)) \le 1$ for any $(u,v)\in E_d$, we have $0 \le \mathcal{X}^i \le 1$
    
    For $n=1$, it is easy to observe that the equality holds. We now prove the result using induction.
    Assume that the statement is true for $n$.
    \begin{equation}
        \implies \sum_{i=1}^n \mathcal{X}^i - \sum_{1\le i<j \le n} \mathcal{X}^i\mathcal{X}^j\le 1-\prod_{i=1}^n \left(1-\mathcal{X}^i\right) \le \sum_{i=1}^n \mathcal{X}^i
    \end{equation}

    When there are $n+1$ agents, we have the following:

    \begin{equation}
        \begin{split}
            &1-\prod_{i=1}^{n+1} \left(1-\mathcal{X}^i\right)
            =1 - \left(1-\mathcal{X}^{n+1}\right)
            \prod_{i=1}^n \left(1-\mathcal{X}^i\right)
            \le 1 - \left(1-\mathcal{X}^{n+1}\right)
            \left(1-\sum_{i=1}^n \mathcal{X}^i\right)\\
            \implies &1-\prod_{i=1}^{n+1} \left(1-\mathcal{X}^i\right)
            \le \sum_{i=1}^{n+1} \mathcal{X}^i
            - \sum_{i=1}^n \mathcal{X}^i\mathcal{X}^{n+1}
            \le \sum_{i=1}^{n+1} \mathcal{X}^i
        \end{split}
    \end{equation}

    Let $\mathcal{M}_1 = \sum_{i=1}^n \mathcal{X}^i$ and $\mathcal{M}_2 = \sum_{1\le i < j \le n}\mathcal{X}^i\mathcal{X}^j$. Then, we have
    \begin{equation}
        \begin{split}
            1-\prod_{i=1}^{n+1} \left(1-\mathcal{X}^i\right)
            &=1 - \left(1-\mathcal{X}^{n+1}\right)
            \prod_{i=1}^n \left(1-\mathcal{X}^i\right)
            \ge 1 - \left(1-\mathcal{X}^{n+1}\right)\left(1-\mathcal{M}_1 + \mathcal{M}_2\right)\\
            \implies 1-\prod_{i=1}^{n+1} \left(1-\mathcal{X}^i\right)
            &\ge \mathcal{X}^{n+1} + \mathcal{M}_1 - \mathcal{M}_2 + \left(\mathcal{X}^{n+1}\right)(\mathcal{M}_2 - \mathcal{M}_1)\\
            &\ge \left(\mathcal{X}^{n+1} + \mathcal{M}_1\right) -\left(\mathcal{M}_2 + \mathcal{M}_1\mathcal{X}^{n+1}\right)
            \ge \sum_{i=1}^{n+1} \mathcal{X}^i - \sum_{1\le i<j \le n+1} \mathcal{X}^i\mathcal{X}^j
        \end{split}
    \end{equation}
    
\end{proof}

Now we prove Theorem \ref{theorem:probability_distribution} using the above result.
\begin{proof}
    We first show the relation between $\mathcal{X}^i(t)$ and the probability values.
    \begin{equation}
        \begin{split}
            \mathcal{X}^i(t) &= \sum_{(u,v) \in E_d} h^i_{(u,v)}(t)V_s((u,v))(w)\\
            \implies \int_0^T \mathcal{X}^i(t) dt &= \int_0^T \sum_{(u,v) \in E_d}V_s((u,v))(w) h^i_{(u,v)}(t) dt\\
            &= \sum_{(u,v) \in E_d}V_s((u,v))(w) \int_0^T h^i_{(u,v)}(t) dt\\
            \implies \sum_{i=1}^n \int_0^T \mathcal{X}^i(t)dt &= nT\sum_{(u,v) \in E_d} P((u,v))V_s((u,v))(w)
        \end{split}
    \end{equation}
    Using Theorem \ref{theorem:sandwich}, we have the following:
    \begin{equation*}
        \sum_{i=1}^n \mathcal{X}^i(t) - \sum_{1\le i<j \le n} \mathcal{X}^i(t)\mathcal{X}^j(t) \le Y_t(w) \le \sum_{i=1}^n \mathcal{X}^i(t)
    \end{equation*}
    
    Hence, we have the following:
    \begin{equation*}
        \int_0^T Y_t(w)\, dt \le nT\sum_{(u,v) \in E_d} P((u,v))V_s((u,v))(w)
    \end{equation*}
    We also have the following:
    \begin{equation*}
        nT\sum_{(u,v) \in E_d} P((u,v))V_s((u,v))(w) 
        \le \sum_{1\le i<j \le n} \int_0^T \mathcal{X}^i(t)\mathcal{X}^j(t)\,dt + \int_0^T Y_t(w)\, dt 
    \end{equation*}

    We can write $Y_t(w)$ in terms of $\mathcal{X}^i(t)$ as follows:
    \begin{equation}
        Y_t(w) = 1- \prod_{i=1}^n\left(1 - \mathcal{X}^i(t)\right)
    \end{equation}
    Using the inequality of geometric mean and arithmetic mean, we can write:
    \begin{equation}
        \left(1 - \frac{\sum_{i=1}^n \mathcal{X}^i(t)}{n}\right)^n \ge \prod_{i=1}^n\left(1 - \mathcal{X}^i(t)\right)
    \end{equation}
    As $0 \le \mathcal{X}^i(t) \le 1$, we have that $0 \le 1 - \left(\sum_{i=1}^n \mathcal{X}^i(t)\right)/n \le 1$.
    \begin{equation}
    \begin{split}
        &\implies 1 - \frac{\sum_{i=1}^n \mathcal{X}^i(t)}{n} \ge \prod_{i=1}^n\left(1 - \mathcal{X}^i(t)\right)\\
        &\implies \frac{\sum_{i=1}^n \mathcal{X}^i(t)}{n} \le Y_t(w)\\
        &\implies T \sum_{(u,v) \in E_d} P((u,v))V_s((u,v))(w) \le \int_0^T Y_t(w) dt
    \end{split}
    \end{equation}
\end{proof}

We can define the probability that a particular node $w$ is visible as any given time $t$ chosen uniformly at random in the interval $[0, T]$ in terms of $Y_t(w)$ as follows:
\begin{equation}
    P_V(w) = \frac{1}{T}\int_0^T Y_t(w) dt
\end{equation}

Let us make binary indicator variables for each node represented as $V_w$. Then the probability that $V_w = 1$ for any time time is $P_V(w)$. \citet{liu2010common} defines the common information using an auxiliary random variable $W$ which makes individual variables independent. We can now write the common information as:
\begin{equation}
    C(\{V_w| w \in V\}) = \inf I(\{V_w| w \in V\}; W)
\end{equation}

We also have the following~\cite{wyner1975common}:
\begin{equation}
\begin{split}
    I(\{V_w| w \in V\};W) &= H(\{V_w| w \in V\}) - H(\{V_w| w \in V\}|W)\\
    \implies I(\{V_w| w \in V\};W) &= H(\{V_w| w \in V\}) - \sum_{w \in V}H(V_w|W)
\end{split}
\end{equation}
where $H$ is the joint entropy of the random variables.

Notice that we can select $W$ such that $P((u,v))$ for all $(u,v) \in E_d$ become independent which also implies that $P_V(w)$ for all $w \in V$ become independent given $W$.
Using Theorem \ref{theorem:probability_distribution}, we have the following:
\begin{equation}
\begin{split}
    &H(V_w|W) \le H(V_w|\{P((u,v)) | (u,v) \in E_d\}) \\
    \implies & H(V_w|W) \le -P_V(w)\log P_V(w)\\
    \implies & H(V_w|W) \le -P_V(w) \log \mathcal{A}(w)\\
    \implies & H(V_w|W) \le -n\mathcal{A}(w)\log \mathcal{A}(w)\\
    \implies & I(\{V_w| w \in V\}) \ge H(\{V_w| w \in V\}) -n\sum_{w \in V} -\mathcal{A}(w)\log \mathcal{A}(w)
\end{split}
\end{equation}
where $\mathcal{A}(w) = \sum_{(u,v) \in E_d} P((u,v))V_s((u,v))(w)$. This implies that minimizing $\sum_{w \in V} -\mathcal{A}(w)\log \mathcal{A}(w)$ will maximize the common information. Note that physically the contribution of the cross terms $\mathcal{X}^i(t)\mathcal{X}^j(t)$ only come into play when two agents are covering the same node at the same time. In sparse settings, the contribution of these cross terms will be small.

%% file: appendix/finegrained_estimator.tex
\section{Fine Grained Estimator}
We now give the proof for Theorem \ref{theorem:close_loop} below.

\begin{proof}
\label{proof:theorem_fine_estimator_path}
    Let us define a path $(\Bar{u}, w_1, w_2, \dots, w_{l-1}, \Bar{v})$ such that $(w_i, w_{i+1}) \in E_d$ for all $0\le i \le l-1$ where $w_0 = \Bar{u}$ and $w_l = \Bar{v}$. As a shorthand, we represent the path as $(\Bar{u}, \dots, \Bar{v})$.

    Let us define a modified edge set $\Bar{E}_d = E_d \cup \{(\Bar{u}, \dots, \Bar{v}\}$ and also have a probability distribution $\Bar{P}((u,v))$ over $\Bar{E}_d$ analogous to the distribution $P((u,v))$ over $E_d$. The only change is that $\Bar{P}((\Bar{u}, \dots, \Bar{v})$ gives the probability that the path is being traversed. We can now define the functions $\Bar{h}^i_{e}(t)$ for all $e \in \Bar{E}_d$. We now define:

    \begin{equation}
        \Bar{P}(e) = \frac{1}{nT}\sum_{i=1}^n \int_0^T \Bar{h}^i_{e}(t) dt \hspace{0.5cm}\forall e \in \Bar{E}_d
    \end{equation}
    and, we will now show that it is indeed a proper distribution.

    Define indicator functions $g^i_{e}(t)$ for all $e \in E_d$ and $ 1\le i \le n$ such that $g^i_{e}(t) = 1$ when agent $i$ is traversing the edge $e$ while the agent is traversing the path $(\Bar{u}, \dots, \Bar{v})$. Obviously, for any edge $e$ that is not in the path, $g^i_{e}(t) = 0$ for all $0 \le t \le T$, $1\le i \le n$. Furthermore,
    \begin{equation}
        h^i_{e}(t) = \Bar{h}^i_{e}(t) + g^i_{e}(t) \hspace{0.5 cm}\forall e \in E_d
    \end{equation}
    \begin{equation*}
    \begin{split}
        \implies& P(e) = \frac{1}{nT}\sum_{i=1}^n \int_0^T \Bar{h}^i_{e}(t) + g^i_{e}(t) \,\,dt \hspace{0.5cm}\forall e \in E_d\\
        \implies& P(e) = \Bar{P}(e) + \frac{1}{nT}\sum_{i=1}^n \int_0^T g^i_{e}(t) \,dt \hspace{0.5cm}\forall e \in E_d\\
        \implies& 1 = \sum_{e \in E_d} \Bar{P}(e) + \frac{1}{nT}\sum_{i=1}^n \int_0^T \sum_{e \in E_d}g^i_{e}(t) \,dt\\
        \implies& 1 = \sum_{e \in E_d} \Bar{P}(e) + \frac{1}{nT}\sum_{i=1}^n \int_0^T \Bar{h}^i_{(\Bar{u},\dots, \Bar{v})}(t) \,dt\\
        \implies& 1 = \sum_{e \in \Bar{E}_d} \Bar{P}(e)
    \end{split}
    \end{equation*}

    Suppose that we define $T_e$ as the time spent on the edge $e \in E_d$ while traversing the path once. It is easy to observe that $\sum_{e \in E_d} T_e = T_{\Bar{u}\Bar{v}}$. We can use the time to get the following relation:
    \begin{equation}
    \begin{split}
        \int_0^T \sum_{i=1}^n g^i_{e}(t) \,dt &= \frac{T_e}{T_{\Bar{u}\Bar{v}}} \int_0^T \sum_{i=1}^n \Bar{h}^i_{(\Bar{u},\dots, \Bar{v})}(t) \,dt \,\, \forall e \in E_d\\
        \implies P(e) &= \Bar{P}(e) + \frac{T_e}{T_{\Bar{u}\Bar{v}}}\Bar{P}((\Bar{u}, \dots, \Bar{v})) \,\, \forall e \in E_d
    \end{split}
    \end{equation}

    Now we show that the loss function over the probabilities on $\Bar{E}_d$ is the same for the probabilities on $E_d$. It is easy to observe that the loss will be the same if we can show that $\sum_{e \in E_d} V_s(e)P(e) = \sum_{e \in \Bar{E}_d} V_s(e)\Bar{P}(e)$. Then, by using Lemma \ref{lemma:visibility_linearity}, we have,

    \begin{equation*}
        \begin{split}
            \sum_{e \in \Bar{E}_d} V_s(e)\Bar{P}(e) &= \sum_{e \in E_d} V_s(e)\Bar{P}(e)
            + V_s((\Bar{u}, \dots, \Bar{v}))\Bar{P}((\Bar{u}, \dots, \Bar{v}))\\
            &= \sum_{e \in E_d} V_s(e)\Bar{P}(e) + \Bar{P}((\Bar{u}, \dots, \Bar{v}))\sum_{e \in E_d} \frac{T_e}{T_{\Bar{u}\Bar{v}}}V_s(e)\\
            &=\sum_{e \in E_d} V_s(e)\left(\Bar{P}(e) + \frac{T_e}{T_{\Bar{u}\Bar{v}}}\Bar{P}((\Bar{u}, \dots, \Bar{v}))\right)\\
            &=\sum_{e \in E_d} V_s(e) P(e)
        \end{split}
    \end{equation*}
\end{proof}

The following is a proof for Theorem \ref{theorem:split_flow}.
\begin{proof}
    \label{proof:theorem_split_flow}
    First we need to show that $\Bar{P}$ and $\Hat{P}$ form a probability distribution over the edges. Let $q=1-p$. Notice that $\Bar{h}^i_{(u,v)}(t) = I(t)h^i_{(u,v)}(t)$ for all $ 0 \le t \le T$ and $(u,v) \in E_d$.

    \begin{equation}
        \begin{split}
            \implies \Bar{P}((u,v)) &= \frac{1}{npT}\int_0^T I(t)h^i_{(u,v)}(t) dt\\
            \implies \sum_{(u,v) \in E_d} \Bar{P}((u,v)) &= \frac{1}{npT}\sum_{i=1}^n\int_0^T I(t) dt\\
            \implies \sum_{(u,v) \in E_d} \Bar{P}((u,v)) &= 1
        \end{split}
    \end{equation}
    
    From the way $\Bar{P}$ and $\Hat{P}$ is defined, it is easy to observe that,
    \begin{equation}
    \label{eq:proof_edge_prob_linearity}
        \begin{split}
            P((u,v)) &= p \Bar{P}((u,v)) + q \Hat{P}((u,v)) \hspace{0.5cm} \forall (u,v) \in E_d\\
            \implies 1 &= p + q \sum_{(u,v) \in E_d} \Hat{P}((u,v))\\
            \implies \Hat{P}((u,v)) &= 1
        \end{split}
    \end{equation}
    Hence, we have that both are probability distributions over the edges. Therefore, we can way that it is possible to decompose $\Bar{P}((u,v))$ into $\Bar{\pi}(u)$ and $\Bar{P}(u \rightarrow v)$, and $\Hat{P}((u,v))$ into $\Hat{\pi}(u)$ and $\Hat{P}(u \rightarrow v)$.
    
    For any $u \in V_d$, we have the following:
    \begin{equation}
    \label{eq:proof_stable_state_linearity}
        \begin{split}
            \implies \pi(u) &= \sum_{\{v|(u,v) \in E_d\}} p\Bar{P}((u,v)) + q\Hat{P}((u,v))\\
            \implies \pi(u) &= p \Bar{\pi}(u) + q \Hat{\pi}(u)
        \end{split}
    \end{equation}

    Let us define the matrix $M$ of size $|V_d| \times |V_d|$ and a vector $x$ of size $|V_d|$ which is indexed by the vertices in $V_d$ such that:
    \begin{equation}
        \begin{split}
            M(v, u) &= \begin{cases}
                P(u \rightarrow v) & (u,v) \in E_d\\
                0 & otherwise
            \end{cases}\\
            x(u) &= \pi(u)
        \end{split}
    \end{equation}
    Similarly, let us also define the following:
    \begin{equation}
        \begin{split}
            \Bar{M}(v, u) &= \begin{cases}
                \Bar{P}(u \rightarrow v) & (u,v) \in E_d\\
                0 & otherwise
            \end{cases} \hspace{1cm}
            \Bar{x}(u) = \Bar{\pi}(u)\\
            \Hat{M}(v, u) &= \begin{cases}
                \Hat{P}(u \rightarrow v) & (u,v) \in E_d\\
                0 & otherwise
            \end{cases}\hspace{1cm}
            \Hat{x}(u) = \Hat{\pi}(u)
        \end{split}
    \end{equation}

    Using equation \ref{eq:proof_stable_state_linearity}, we have that $x = p\Bar{x} + q\Hat{x}$, and equation \ref{eq:proof_edge_prob_linearity} implies that $Mx = p\Bar{M}\Bar{x} + q\Hat{M}\Hat{x}$.
    \begin{equation}
        \begin{split}
            \implies x - Mx &= p(\Bar{x} -  \Bar{M}\Bar{x}) + q(\Hat{x} -  \Hat{M}\Hat{x})\\
            \implies (I-M)x &= p(I-\Bar{M})\Bar{x} + q(I-\Hat{M})\Hat{x}
        \end{split}
    \end{equation}
    Hence, it is sufficient to constraint that $\Bar{M}\Bar{x} = \Bar{x}$ and $\Hat{M}\Hat{x} = \Hat{x}$ which is same as saying:
    \begin{equation*}
            \sum_{\{v|(u,v) \in E_d\}} \Hat{\pi}(u)\Hat{P}(u \rightarrow v) = \Hat{\pi}(v) \text{\hspace{0.3cm} and \hspace{0.3cm}}
            \sum_{\{v|(u,v) \in E_d\}} \Bar{\pi}(u)\Bar{P}(u \rightarrow v) = \Bar{\pi}(v)            
    \end{equation*}
    \todo{Add the countable partition proof?}
\end{proof}

%% file: appendix/large_environment.tex
\section{Further Analysis of Small Environment}
\label{app:further_analysis}

We conduct further analysis by calculating the inter-cover time for nodes, the minimum cover frequency, and the standard deviation of cover frequency over different nodes. It should be noted that all these metrics have their own set of problems. These metrics are a good measure for static coverage; however, they face multiple issues for the dynamic nature of coverage in this work. Another nuance to consider is that we are not explicitly creating paths. Some methods create closed loop paths where the metrics make sense with the proposed algorithm. In contrast, our objective is to evaluate how good the learned distribution is.

To illustrate the issue, we calculate the variance of these metrics over seeds and nodes. Following table shows the various metrics for the small environment:

\begin{table*}[ht]
\centering

\begin{tabular}{lcccc}
\hline
\textbf{Method} & 
\textbf{Mean Seeds} & 
\textbf{Std Seeds} & 
\textbf{Mean Seeds} & 
\textbf{Std Seeds} \\

\textbf{} & 
\textbf{Mean Nodes} & 
\textbf{Mean Nodes} & 
\textbf{Std Nodes} & 
\textbf{Std Nodes} \\

\hline

Random Agent 
& 23.7541 & 0.0835 & 124.6502 & 7.1879 \\

Sample Agent 
& 39.9655 & 0.5161 & 265.0988 & 63.2053 \\

Comm Sample Agent 
& 48.9734 & 0.4433 & 289.5213 & 79.1258 \\

Frontier Agent 
& 12.4113 & 0.0367 & 97.9757 & 4.2560 \\

Comm Frontier Agent 
& 12.3572 & 0.0282 & 101.1581 & 3.0090 \\

Optimal Agent 
& 157.6597 & 0.0167 & 68.3205 & 0.6030 \\

OptRandom Agent 
& 12.5551 & 0.0063 & 12.0341 & 0.0119 \\

\hline
\end{tabular}

\vspace{0.5cm}

\begin{tabular}{lcc}
\hline
\textbf{Method} & 
\textbf{Min Cover Freq (Mean $\pm$ Std)} & 
\textbf{Avg Node Cover Std} \\
\hline

Random Agent 
& $449.4 \pm 125.13$ & 3041.19 \\

Sample Agent 
& $286.7 \pm 115.96$ & 11330.91 \\

Comm Sample Agent 
& $216.6 \pm 120.81$ & 11551.92 \\

Frontier Agent 
& $2328.3 \pm 63.81$ & 1722.55 \\

Comm Frontier Agent 
& $2345.0 \pm 111.82$ & 1700.66 \\

Optimal Agent 
& $56.8 \pm 5.44$ & 12901.76 \\

OptRandom Agent 
& $964.9 \pm 27.08$ & 2143.25 \\

\hline
\end{tabular}

\caption{Comparison of agent methods on inter-cover time and coverage uniformity metrics for the Small environment. The first block reports inter-cover time statistics across seeds and nodes (mean and standard deviation of node-wise means and standard deviations), while the second block summarizes coverage performance via the minimum cover frequency across nodes (mean $\pm$ standard deviation over seeds) and the average standard deviation of node visit frequencies. Lower inter-cover times indicate faster revisitation, whereas lower variability metrics reflect more uniform coverage.}
\label{tab:cover_metrics_split_rows}

\end{table*}

For the inter-cover metric, the mean over seeds and nodes can be misleading because our objective is to have uniform coverage. The standard deviation along nodes tells how far the method is deviating from uniform cover over nodes. An interesting observation is that uniform cover affects minimum frequency at which a node is covered. This trade-off might be of interest to those who want to optimize both unifrom coverage and visitation frequency.

\section{Large Environment Experiments}
\label{app:large_environment}
The large environment is a $70 \times 70$ grid-like structure with an uneven but smooth topology shown in Figure \ref{fig:large_environment}. Two stark differences from the small environment are, 1) there are no walls that suddenly clip the visibility, and 2) the topology does not restrict the movement of agents. We do simulations for homogeneous teams with $30$ agents with unidirectional visibility.

Regarding scalability, the computational bottleneck of BOIL is the iterative PageRank-style update, which scales linearly with the number of edges $O(|E|)$. This allows the method to remain tractable even as the environment size $|V_d|$ grows, provided the graph connectivity remains sparse (as is typical in topographic maps).

\begin{figure}[H]
    {\includegraphics[scale=0.4]{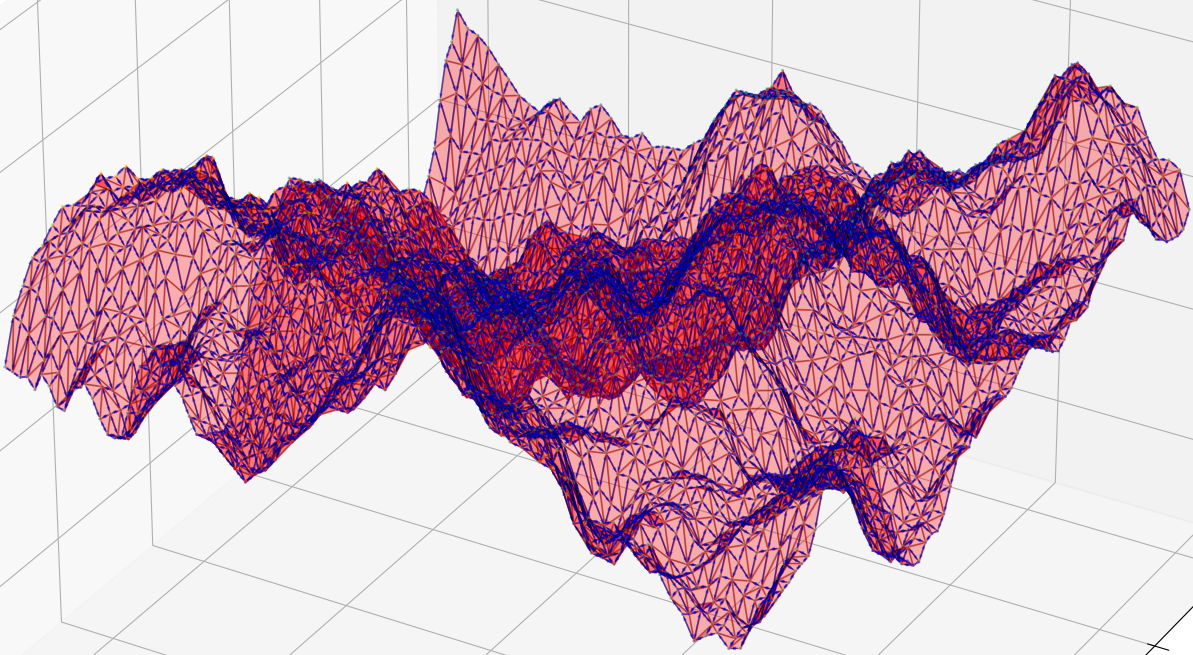}}
    {\caption{The large environment has a variety of features that are observed in real world topographies. It has multiple different hills and valleys resulting in a complex situation to analyze even for human experts.}\label{fig:large_environment}}
\end{figure}

Figure \ref{fig:large_environment_vis_counts} shows the cumulative counts and Figure \ref{fig:large_environment_convergence} shows the total variation distance of the distribution approximated by the trajectories and the distribution found using BOIL process.

\begin{figure}[H]
    \centering
    \includegraphics[scale=0.29]{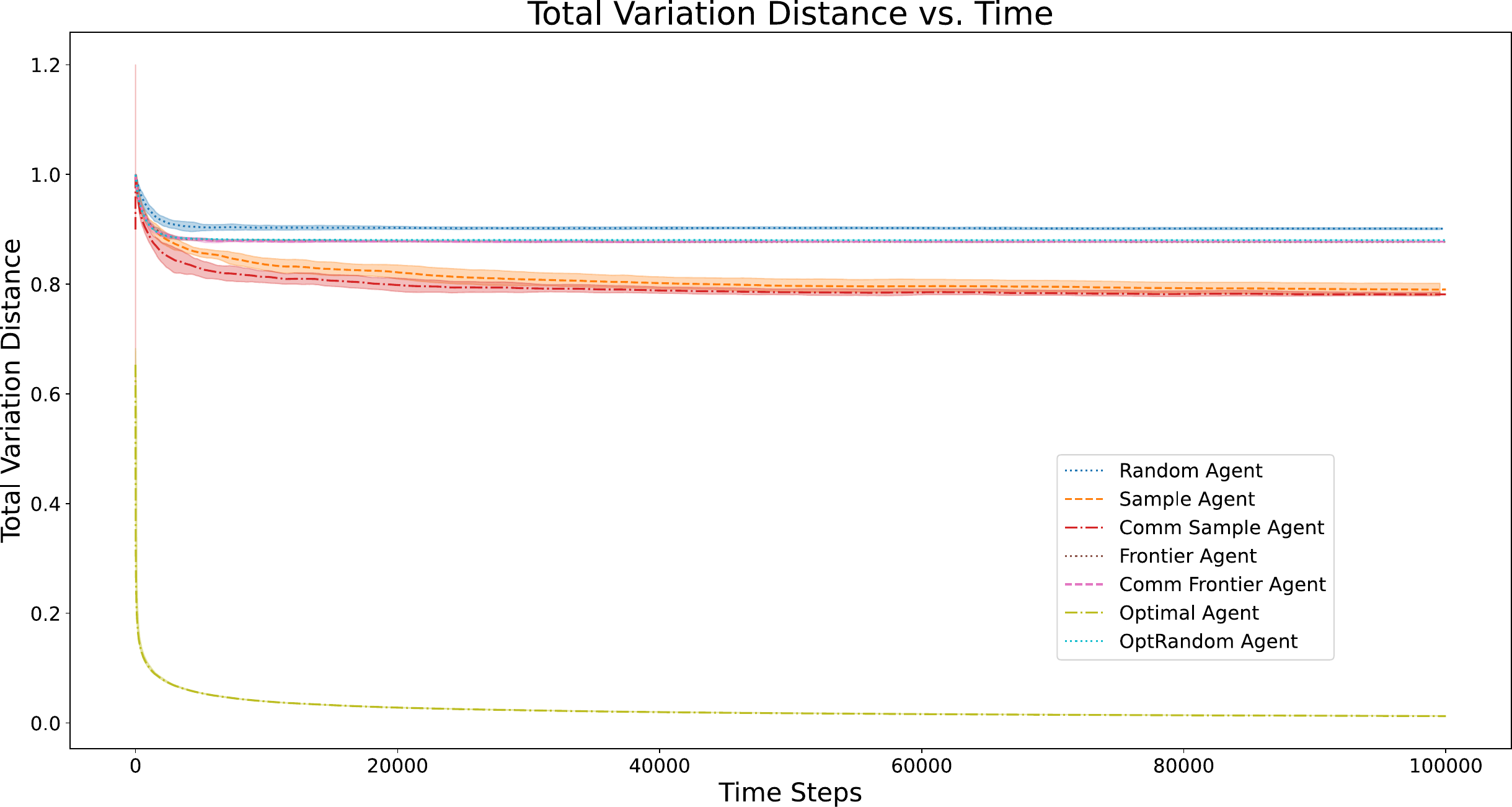}
    \caption{The figure depicts Time steps vs. Total variation distance between the approximate distribution of the agent trajectories and the distribution found using the BOIL process for the large environment. \textbf{BOIL Distribution Strategies:} \textit{Sample Agent, Comm Sample Agent, Optimal Agent}}
    \label{fig:large_environment_convergence}
\end{figure}

Let us initially examine Figure \ref{fig:large_environment_convergence} to assess the convergence rate. It becomes evident that random walks and frontier-based exploration strategies converge to a similar distance value. Conversely, Sample and Comm Sample strategies struggle to achieve convergence. This disparity is further underscored by the visibility counts depicted in Figure \ref{fig:large_environment_vis_counts}. The challenge of attaining convergence to the desired distribution becomes apparent as the environment's scale and complexity increase. However, the visibility counts also suggest that coverage remains satisfactory once agents reach the distribution obtained through the BOIL process. Thus, despite the inherent difficulty in achieving convergence, the obtained distribution is effective for ensuring adequate coverage.

\begin{figure}[H]
    {\includegraphics[scale=0.25]{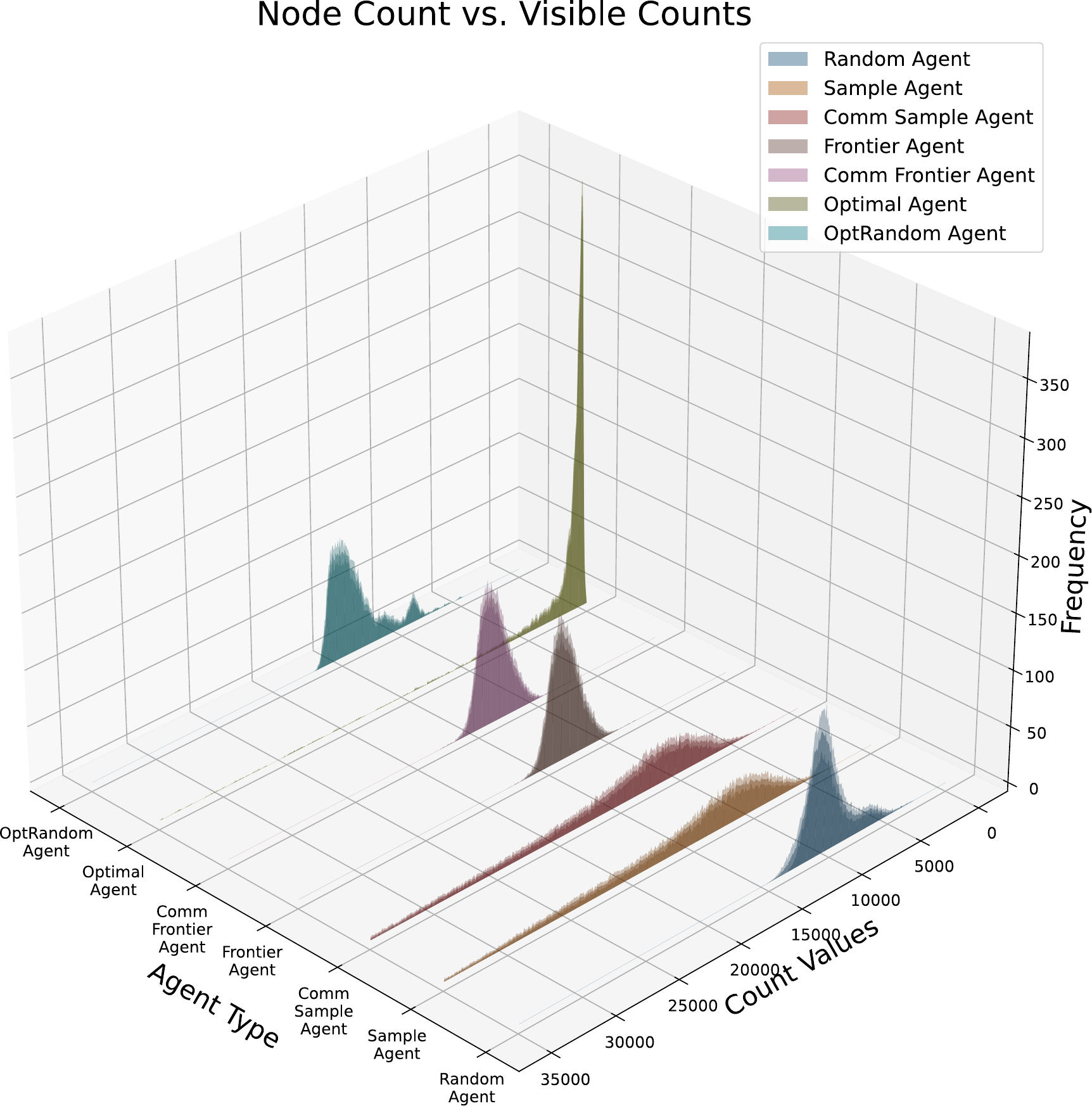}}
    {\caption{Similar to Figure \ref{fig:small_environment_vis_counts}, this figure shows the same type of data for the large environment. Z-axis illustrates the counts of visible nodes corresponding to values on the Y-axis, while the X-axis denotes the types of homogeneous agents considered in the experiment. Curves' opacity signifies variations in values, with lower opacity indicating the upper end of the variance and higher opacity representing the lower end for frequencies. The central divider indicates the mean frequencies. \textbf{BOIL Distribution Strategies:} \textit{Sample Agent, Comm Sample Agent, Optimal Agent}}\label{fig:large_environment_vis_counts}}
\end{figure}

It should be noted that despite the apparent plateauing of convergence, closer inspection of the raw values reveals a downward gradient for both the small and large environments for Sample and Comm Sample strategies. This observation aligns with theoretical expectations, as the sampling process theoretically guarantees convergence to the desired distribution. Furthermore, bidirectional paths in the environment may contribute to oscillatory behavior, given the simplicity of the path-planning process.

A closer look at Figure \ref{fig:large_environment_vis_counts} suggests that the distributions of Sample and Comm Sample agents exhibit a skewed peak resembling that of Optimal agents. This indicates that despite encountering challenges in achieving full convergence, these strategies strive to emulate the optimal distribution pattern.

Table below provides similar inter-cover and minimum cover frequency related metrics for the large environment.

\begin{table*}[ht]
\centering

\begin{tabular}{lcccc}
\hline
\textbf{Method} & 
\textbf{Mean Seeds} & 
\textbf{Std Seeds} & 
\textbf{Mean Seeds} & 
\textbf{Std Seeds} \\

\textbf{} & 
\textbf{Mean Nodes} & 
\textbf{Mean Nodes} & 
\textbf{Std Nodes} & 
\textbf{Std Nodes} \\

\hline

Random Agent 
& 10.5195 & 0.0072 & 44.9495 & 0.6216 \\

Sample Agent 
& 7.7228 & 0.0464 & 50.9775 & 1.9425 \\

Comm Sample Agent 
& 7.9123 & 0.0265 & 53.5128 & 1.7902 \\

Frontier Agent 
& 7.8568 & 0.0081 & 39.2050 & 0.4781 \\

Comm Frontier Agent 
& 7.8268 & 0.0067 & 42.5300 & 0.7787 \\

Optimal Agent 
& 119.2871 & 0.0279 & 77.6937 & 0.1387 \\

OptRandom Agent 
& 7.8132 & 0.0029 & 7.3050 & 0.0036 \\

\hline
\end{tabular}

\vspace{0.5cm}

\begin{tabular}{lcc}
\hline
\textbf{Method} & 
\textbf{Min Cover Freq (Mean $\pm$ Std)} & 
\textbf{Avg Node Cover Std} \\
\hline

Random Agent 
& $3048.7 \pm 163.81$ & 1550.97 \\

Sample Agent 
& $3446.1 \pm 352.05$ & 8167.65 \\

Comm Sample Agent 
& $3206.7 \pm 606.49$ & 8699.91 \\

Frontier Agent 
& $7804.2 \pm 234.52$ & 1051.95 \\

Comm Frontier Agent 
& $7745.2 \pm 218.68$ & 1047.37 \\

Optimal Agent 
& $105.1 \pm 8.76$ & 12400.52 \\

OptRandom Agent 
& $4898.4 \pm 72.15$ & 1740.42 \\

\hline
\end{tabular}

\caption{Comparison of agent methods on inter-cover time and coverage uniformity metrics for the Large environment. The first block reports inter-cover time statistics across seeds and nodes (mean and standard deviation of node-wise means and standard deviations), while the second block summarizes coverage performance via the minimum cover frequency across nodes (mean $\pm$ standard deviation over seeds) and the average standard deviation of node visit frequencies. Lower inter-cover times indicate faster revisitation, whereas lower variability metrics reflect more uniform coverage.}
\label{tab:cover_metrics_split_rows_large}

\end{table*}